\documentclass[a4paper,USenglish]{article}
\usepackage[T1]{fontenc}
\usepackage{graphicx}
\usepackage{booktabs}
\usepackage[misc]{ifsym}

\usepackage{mathtools} \usepackage{amsthm}
\usepackage{amsmath}
\usepackage{amsbsy}
\usepackage{amssymb}
\usepackage{amsfonts}
\usepackage{nicefrac}
\usepackage{bm}
\usepackage{bbm}
\usepackage{adjustbox}
\usepackage{multirow}
\usepackage{tabularx}   
\usepackage{subcaption}
\usepackage[ruled,noend,vlined]{algorithm2e}
\usepackage{pgf,tikz}
\usetikzlibrary{shapes, arrows, automata, positioning, calc, matrix}
\usepackage[auth-sc,affil-it]{authblk}

\newtheorem{theorem}{Theorem}
\newtheorem{lemma}[theorem]{Lemma}

\newtheorem*{claim*}{Claim}
\newtheorem*{fact*}{Fact}

\newcommand{\NP}{$\textsf{NP}$}
\newcommand{\NPPP}{$\textsf{NP}^\textsf{PP}$}
\newcommand{\EX}{$\textsc{ex}$}

\newcommand{\defeq}{\stackrel{\text{def}}{=}} 
\renewcommand{\vec}[1]{\bm{#1}}

\newcommand{\size}[1]{|#1|}

\newcommand{\Body}[1]{\mathit{body}{(#1)}}
\newcommand{\Head}[1]{\mathit{head}{(#1)}}

\title{Learning Model Agnostic Explanations via Constraint Programming}

\author{Frederic Koriche}
\author{Jean-Marie Lagniez}
\author{Stefan Mengel}
\author{Chi Tran}

\affil{Univ.~Artois, CNRS, Centre de Recherche en Informatique de Lens (CRIL), France}

\begin{document}
    
    \maketitle              

\begin{abstract}
    Interpretable Machine Learning faces a recurring challenge of explaining the predictions made by opaque classifiers such as ensemble models, 
    kernel methods, or neural networks in terms that are understandable to humans. 
    When the model is viewed as a black box, the objective is to identify a small set of features that jointly determine the black box response with minimal error. 
    However, finding such model-agnostic explanations is computationally demanding, as the problem is intractable even for binary classifiers. 
    In this paper, the task is framed as a Constraint Optimization Problem, where the constraint solver seeks an explanation of minimum error and bounded size for 
    an input data instance and a set of samples generated by the black box. From a theoretical perspective, this constraint programming approach 
    offers PAC-style guarantees for the output explanation. We evaluate the approach empirically on various datasets and show that
    it statistically outperforms the state-of-the-art heuristic \textsc{Anchors} method.
\end{abstract}

\section{Introduction}

With the increasing influence of machine learning systems in our daily lives, 
there is a pressing need to comprehend the reasoning behind their predictions. 
However, the current black-box nature of predictive models such as, for example, 
ensemble methods, kernel techniques or neural networks, often leaves users bewildered. 
In order to address this challenge, the field of Interpretable Machine Learning (IML) has emerged, 
which focuses on developing new learning and explanation techniques to make 
predictive models more transparent and comprehensible \cite{Burkart.JAIR.2021,Molnar.Book.2022}.

A common task in IML is explaining the predictions made by a classifier 
on data instances in terms that are easy for humans to understand. 
For example, when evaluating a personal loan application based on tabular data about the applicant, 
features such as a stable income, a low debt-to-income ratio, and the availability of a co-signer 
may have a collective impact on the loan approval. 
When the classifier is viewed as a black box with only query access to some instance-to-label function $f$, 
the set of features $S$ used to explain the output $f(\vec x)$ of an instance $\vec x$ is referred to as 
\emph{model agnostic explanation} \cite{Lundberg.NeurIPS.2017,Ribeiro.KDD.2016,Ribeiro.AAAI.2018}. 

One important aspect of model agnostic explanations is their \emph{flexibility}.  
Such explanations can indeed be provided without any knowledge of the internal structure and parameters of the predictive model. 
This property is not only useful for deciphering the predictions made by complex models but also 
for maintaining confidentiality in cases where the internal components of the model cannot be revealed.
However, flexibility is not the only criterion for assessing the quality of explanations.
Explaining the predictions of black-box models involves balancing two other important criteria. 

The first is the \emph{precision}, or \emph{fidelity} \cite{Molnar.Book.2022}, of the explanation in approximating the model's predictions. 
An explanation $S$ for an instance $\vec x$ and a model $f$ is considered \emph{abductive} if changes in feature values outside of the explanation $S$ do not affect the model's output $f(\vec x)$. 
When $f$ is a Boolean function, the conjunction of features in the abductive explanation $S$ can be viewed as an implicant of $f$ that covers the instance $\vec x$. 
Abductive explanations have received significant attention in the literature due to their perfect fidelity \cite{Audemard.AAAI.2022,Cooper.AI.2023,Darwiche.ECAI.2020,Ignatiev.AAAI.2022,Ignatiev.AAAI.2019}.

The second criterion to consider is the \emph{conciseness} of the explanation, which is limited by the human ability to reason about multiple interacting factors. 
As Miller conjectured in 1956 \cite{Miller.PR.1956}, this cognitive limit is of seven plus or minus two elements, 
which has been confirmed by numerous experiments in cognitive science (see e.g. \cite{Saaty.MCM.2003}). 
Abductive explanations suffer from a significant weakness in this regard, as their size is often uncontrollable. 
Even the smallest abductive explanation may require too many features to be logically sound and comprehensible to humans. 
Therefore, finding a balance between precision and conciseness is crucial when generating explanations for black-box models.

The concept of \emph{probabilistic explanation}, introduced by Blanc et al. \cite{Blanc.NeurIPS.2021} and Waeldchen et al. \cite{Waeldchen.JAIR.2021}, 
captures this balance. Informally, the precision error $\epsilon_{f,\vec x}(S)$ of an explanation $S$ for an instance $\vec x$ and a model $f$ is defined by the probability 
that a random instance $\vec z$ is classified differently from $\vec x$ by $f$, when $\vec z$ and $\vec x$ 
agree on all the features in $S$. In other words, $\epsilon_{f,\vec x}(S)$ is the probability of making a mistake in predicting 
the black-box response $f(\vec x)$ using only the features in $S$. 
With this notion in hand, the goal of the present study is to find a model-agnostic explanation $S$ that is no larger than a small constant $k$, 
while minimizing the error $\epsilon_{f,\vec x}(S)$ for a given input $\vec x$ and a black-box model $f$.  

This optimization task is very computationally intensive. 
Even in the model-specific case where we have a Boolean circuit representation of $f$, 
finding the minimizer $S$ of $\epsilon_{f,\vec x}(S)$ while satisfying the cardinality constraint $\size{S} \leq k$ is \NPPP-hard \cite{Waeldchen.JAIR.2021}. 
This difficulty persists even in the case where $f$ is represented by a binary decision tree, as the problem remains \NP-hard \cite{Arenas.NeurIPS.2022}.

\subsection{Contributions}

This paper introduces a framework for learning explanations that are independent of the model, consisting of two main ideas. 
As any explanation can be represented as an \emph{if-then} rule \cite{Ribeiro.AAAI.2018}, we consider a hypothesis class 
that includes all rules whose body is a subset of features of the data instance $\vec x$ to be explained, 
and whose head is the prediction $f(\vec x)$. We first demonstrate that an efficient agnostic PAC learner for this hypothesis class 
can lead to an efficient model-agnostic explainer. However, since this hypothesis class is equivalent to the class of monotone monomials, 
the learning problem remains \NP-hard. Nevertheless, we next show that it can be formulated as a Constraint Optimization Problem (COP) using a 
few linear constraints and channeling constraints. By combining these two ideas, the optimal value attained by the constraint solver 
provides a bounded approximation of the objective value for the precision function $\epsilon_{f,\vec x}(\cdot)$.

From an empirical standpoint, our approach is compared with the well-known model-agnostic explainer \textsc{Anchors} \cite{Ribeiro.AAAI.2018}. 
By conducting experiments on various datasets, we show that our approach outperforms \textsc{Anchors} in terms of 
precision, while using a reasonable amount of time for the solver.

\subsection{Related Work}

In the field of Interpretable Machine Learning, researchers have proposed various techniques to understand the behavior of machine learning models. 
\emph{Post hoc explanations} are one such technique that aims to provide insights into how a model works without affecting its performance. 
These explanations give additional knowledge about the prediction $f(\vec x)$ made by a model $f$ on a specific data instance $\vec x$. 
Common types of post hoc explanations are \emph{feature-based explanations} \cite{Lundberg.NeurIPS.2017,Ribeiro.AAAI.2018}, 
which interpret the prediction $f(\vec x)$ by identifying the most relevant features of the instance $\vec x$, 
and \emph{example-based explanations} \cite{Karimi.AISTATS.2020,Koh.ICML.2017}, which improves the comprehension of  
$f(\vec x)$ using counterfactual or influential examples. Feature-based explanations, examined in this study, 
can be categorized into two types: \emph{model-specific} and \emph{model-agnostic} explanations. 

\subsubsection{Model-Specific Explanations.} 

By exploiting the structure and parameters of the predictive model $f$, 
model-specific approaches aim to provide precise and mathematically verifiable explanations.
Notably, when $f$ is a Boolean classifier, a common explanation for predicting the output $f(\vec x)$ of an input instance $\vec x$ 
is a subset-minimal collection of features $S$ such that the restriction $\vec x_S$ of $\vec x$ to $S$ determines $f(\vec x)$. 
Such an \emph{abductive} explanation \cite{Ignatiev.AAAI.2019}, also called \emph{sufficient reason} \cite{Darwiche.ECAI.2020}, 
is perfectly accurate since $\vec x_S$ is a prime implicant of $f$ that covers $\vec x$ \cite{Shih.IJCAI.2018}.
Although finding an abductive explanation is generally an \NP-hard problem, tractable cases have been identified 
for various model classes including, among others, decision trees \cite{Audemard.DKE.2022,Huang.KR.2021,Izza.JAIR.2022}, 
Naive Bayes classifiers \cite{MarquesSilva.NeurIPS.2020}, and monotone threshold functions \cite{Cooper.AI.2023}. 
Actually, even for intractable classes such as decision lists and random forests, 
empirical results indicate that abductive explanations can be quickly found in practice using constraint-based approaches 
\cite{Audemard.AAAI.2022,Ignatiev.AAAI.2022,Ignatiev.SAT.2021}.

In order to overcome the uncontrollable size of abductive explanations, 
model-specific approaches have recently focused on \emph{probabilistic explanations}, 
which convey a natural trade-off between conciseness and precision \cite{Bounia.UAI.2023,Izza.JAR.2023}. 
In simple terms, a $(k,\varepsilon)$-probabilistic explanation for a classifier $f$ and an instance $\vec x$ 
is a subset $S$  of features such that $\size{S} \leq k$ and $\epsilon_{f,\vec x}(S) \leq \varepsilon$. 
Unfortunately, finding probabilistic explanations is generally more difficult than finding abductive explanations. 
Indeed, the problem of deciding whether a Boolean circuit representation of $f$ admits a $(k,\varepsilon)$-probabilistic 
explanation for an instance $\vec x$ is complete for \NPPP \cite{Waeldchen.JAIR.2021}, a complexity class beyond the reach of modern solvers.  

\subsubsection{Model-Agnostic Explanations.} 

Unlike model-specific explanations, model-agnostic explanations do not rely on any assumptions 
about the inner workings of the predictive model $f$. 
Instead, $f$ is treated as a black box and accessed through membership queries.
In post-hoc explanations, model-agnostic techniques seek a surrogate that approximates 
the model $f$ in some local neighborhood of the instance to be explained $\vec x$. 
This neighborhood is explored using sampling and perturbation strategies. 
When the model is a classifier, the surrogate provided by the \textsc{Lime} approach \cite{Ribeiro.KDD.2016} and its variants 
\cite{Dhurandhar.NeurIPS.2022,Zhao.UAI.2021} takes the form of a linear threshold function. 
Similarly, in the \textsc{Shap} approach \cite{Lundberg.NeurIPS.2017},  
the returned explanation is a linear function where each entry approximates 
the Shapley value of the corresponding feature that contributes to the outcome. 
Arguably, the \textsc{Anchors} approach \cite{Ribeiro.AAAI.2018} is the most relevant work to our study. 
It represents the surrogate of the model $f$ in the neighborhood of the instance $\vec x$ as an easy-to-understand \emph{if-then} rule. 
This rule can be described as a set $S$ of features which jointly determine the outcome $f(\vec x)$.

Although model-agnostic approaches are appealingly flexible, most of them are heuristic, 
meaning that they do not provide any theoretical guarantees regarding the quality or size of the 
explanations they generate. In practice, these approaches often return incorrect explanations \cite{Ignatiev.2020}, 
which can make them even less reliable. Such observations highlight the necessity of alternative techniques that preserve 
the flexibility of model agnosticism but provide precision and conciseness guarantees.

To the best of our knowledge, only Blanc et al. \cite{Blanc.NeurIPS.2021} have researched this area. 
Their objective is to find approximations of $(k,\varepsilon)$-probabilistic explanations for random data instances, 
using only query access to the model $f$. Based on some techniques in implicit learning, their result provides 
PAC-style guarantees on the quality and size of the approximate explanations. 
However, their approach is difficult to apply in practice since its runtime complexity is prohibitive. 
Our framework, in contrast, provides similar guarantees, using a restricted number of calls to an \NP-oracle, 
implemented by a constraint solver.

\section{Preliminaries}

In model-agnostic approaches to Interpretable Machine Learning, the predictive model is treated as a black box. 
However, if we want to explain the model's response to input data in a way that humans can understand, 
we require a set of \emph{interpretable features} \cite{Ribeiro.KDD.2016,Ribeiro.AAAI.2018}. 
For example, when explaining a loan application, attributes such as ``stable income'', ``low debt-to-income ratio'', 
and ``proof of address'' can be considered as interpretable features. Any instance to be explained is 
a representation indicating the presence or absence of each of these features. In classification scenarios where 
the model's output is a label, the goal of the explainer is to pinpoint a few interpretable features that 
collectively determine this label with high precision.

In more formal terms, let $[d]$ denote the set $\{1,\cdots,d\}$ of interpretable features.
The black-box models under consideration in this study are binary classifiers of the form $f: \{0,1\}^d \rightarrow \{0,1\}$, 
Any input instance $\vec x$ is a $d$-dimensional Boolean vector, where $x_j$ denotes the (Boolean) value 
of the $j$th interpretable feature in $[d]$. 
A \emph{partial instance} is a vector $\vec u \in \{0,1,*\}^d$, where $u_j = *$ indicates that the $j$th feature of 
$\vec u$ is left undefined.
An instance $\vec x$ is \emph{covered} by $\vec u$, if $x_j = u_j$ for all features $j \in [d]$ such that $u_j \neq *$. 
For a subset $S \subseteq [d]$, the \emph{restriction} of $\vec x$ to $S$, denoted $\vec x_S$, is the partial instance 
such that, for each $j \in [d]$, $(x_S)_j = x_j$ if $j \in S$, and $(x_S)_j = *$ otherwise. 
Clearly, any instance $\vec z \in \{0,1\}^d$ is covered by $\vec x_S$ if and only if $\vec z_S = \vec x_S$.  

Given a size limit $k \geq 0$, a \emph{$k$-explanation} for a data instance $\vec x \in \{0,1\}^d$ with respect to a black-box model $f: \{0,1\}^d \rightarrow \{0,1\}$ 
consists in a set of features $S \subseteq [d]$ such that $\size{S} \leq k$. 
For a probability distribution $\mathcal D$ over $\{0,1\}^d$, the \emph{precision error} of $S$ is given by
\begin{align}
\label{def.precision}
\epsilon_{f,\vec x,\mathcal D}(S) \defeq \mathbb P_{\vec z \sim \mathcal D} [f(\vec z) \neq f(\vec x) \mid \vec z_S = \vec x_S] 
\end{align}
where $\vec z \sim \mathcal D$ indicates that $\vec z$ is chosen according to the distribution over $\mathcal D$.
As mentioned earlier, $\epsilon_{f,\vec x,\mathcal D}(S)$ represents the probability of making an incorrect prediction 
when using the partial instance $\vec x_S$ instead of the complete instance $\vec x$. 
It is worth noting that the definition of precision error, originally introduced in \cite{Blanc.NeurIPS.2021,Waeldchen.JAIR.2021},
is generalized by Equation (\ref{def.precision}) where $\mathcal D$ is not restricted to the uniform distribution over $\{0,1\}^d$. 
By allowing the use of an arbitrary distribution $\mathcal D$ in our study, 
we can explore interesting cases where, for example, the precision error is measured on the neighborhood 
of the instance to explain. In the rest of this text, we will exclude $\mathcal D$ from the subscript of (\ref{def.precision}) 
and other definitions whenever it is understood from the context.

Since $f$ is a black box with only query access, we cannot evaluate the precision of $S$ exactly, unless we ask for an exponential number 
of membership queries which is clearly infeasible for nontrivial values of $d$. 
Instead, we approximate the precision of $S$ using an \emph{example oracle} $\EX(f,\mathcal D)$. 
On each call, $\EX(f,\mathcal D)$ returns a labeled instance $(\vec z, f(\vec z))$, where $\vec z$ is drawn at random according to $\mathcal D$. 

With these notions in hand, the problem examined in this study is the following. Given a black-box model $f$, a probability distribution $\mathcal D$, 
and a size limit $k$, for each data instance $\vec x$ picked at random according to $\mathcal D$, find a $k$-explanation $S$ for $f$ and $\vec x$ 
that minimizes the objective function $\epsilon_{f,\vec x}(\cdot)$. In other words, our goal is to minimize  
\begin{align}
    \label{eq.problem}
    \mathbb E_{\vec x \sim \mathcal D} [\epsilon_{f,\vec x}(S)] \;\mbox{ subject to }\; \size{S} \leq k
\end{align}
The performance of a model-agnostic explanation algorithm for a model $f$ is measured by its expected precision error over $\mathcal D$
and the number of calls to the example oracle to $\EX(f,\mathcal D)$. 

\section{The Framework}

As we mentioned earlier, providing concise and precise explanations for arbitrary classifiers is a challenging task. 
Therefore, we cannot hope to offer polynomial-time algorithms for such a problem. 
So, in order to address this challenge, our framework relies on two key concepts. First, we approximate the explanation task 
by learning an \textsc{if-then} rule where the body is a subset of literals occurring in the instance $\vec x$ to explain,
and the head is the predicted label $f(\vec x)$. Next, we solve this learning task through constraint optimization. 
The low generalization error of the rule provided by the solver is the connection between these two concepts, 
which is used to derive high-precision explanations.

In what follows, we use $\mathbbm 1[\cdot]$ to denote the indicator function. Specifically, for an event $E$, $\mathbbm 1[E] = 1$ if $E$ is true,
and $\mathbbm 1[E] = 0$ otherwise. 

\subsection{From Explanations to Rules}

For a set $\{X_j\}_{j=1}^d$ of Boolean variables, a \emph{literal} is an expression of the form $X_j = v_j$, where $v_j$ is a Boolean value in $\{0,1\}$,
and a \emph{monomial} is a conjunction of literals. An \textsc{if-then} \emph{rule} $R$ consists in a monomial, denoted $\Body{R}$, and a 
label in $\{0,1\}$, denoted $\Head{R}$. 
A \emph{$k$-monomial} is a monomial of size at most $k$, and a \emph{$k$-rule} is a rule whose body is a $k$-monomial. 
A rule $R$ is \emph{fired} on an instance $\vec z \in \{0,1\}^d$ 
if its body is true for $\vec z$, that is, $v_j = z_j$ for each literal $X_j = v_j$ occurring in $\Body{R}$. 
In that case, the label assigned to $\vec z$ is $\Head{R}$.
On the other hand, if the body of $R$ is false for $\vec z$, then the corresponding label is $1 - \Head{R}$. 
A Boolean hypothesis $r: \{0,1\}^d \rightarrow \{0,1\}$ is \emph{representable} by a rule $R$ 
if the following condition holds: for any $\vec z \in \{0,1\}^d$, $r(\vec z) = \Head{R}$ if and only if $\Body{R}$ is true for $\vec z$.

Based on these standard notions, any explanation $S$ for a given data instance $\vec x \in \{0,1\}^d$ and a black-box model $f: \{0,1\}^d \rightarrow \{0,1\}$ 
can be described by an \textsc{if-then} rule where the body is the set of literals $X_j = x_j$ such that $j \in S$, and the head is the prediction $f(\vec x)$. 
By construction, this rule is fired on any instance $\vec z \in \{0,1\}^d$ if and only if $\vec x_S$ covers $\vec z$, 
that is, $\vec x_S = \vec z_S$. Thus, the \textsc{if-then} rules of interest in this framework are those whose body is true for $\vec x$ and whose head is $f(\vec x)$.

From this perspective, let $\mathcal R_{f,\vec x}$ denote the class of all Boolean hypotheses $r: \{0,1\}^d \rightarrow \{0,1\}$ 
representable by a rule $R$, where $\Body{R}$ is a monomial over the set of literals $\{X_j = x_j\}_{j = 1}^d$ and $\Head{R} = f(\vec x)$.
It is important to note that any hypothesis $r \in \mathcal R_{f,\vec x}$ agrees with $f$ on $\vec x$, 
that is, $r(\vec x) = f(\vec x)$. The \emph{(zero-one) loss} of $r$ with respect to $f$ is defined as
\begin{align}
\label{eq.true.loss}
    L_f(r) = \mathbb P_{\vec z \sim \mathcal D} [r(\vec z) \neq f(\vec z)]  
\end{align}
where $\vec z$ is again chosen randomly according to the distribution $\mathcal D$.

By $\mathcal R_{f,\vec x,k}$ we denote the restriction of $\mathcal R_{f,\vec x}$ to 
the hypotheses representable by $k$-rules.
The following result shows that algorithms for learning \textsc{if-then} rules 
in $\mathcal R_{f,\vec x, k}$ can be used to explain predictions.

\begin{lemma}
\label{lemma.learning}
Let $f: \{0,1\}^d \rightarrow \{0,1\}$ be a black-box model, and $\mathcal D$ be a probability distribution over $\{0,1\}^d$.
Suppose there exists an algorithm with the following property: 
given $k \in \mathbb N$, and an instance $\vec x \in \{0,1\}^d$, 
the algorithm returns a rule representing a hypothesis $r \in \mathcal R_{f,\vec x,k}$.
Then, the rule representing $r$ determines a $k$-explanation $S$ for $\vec x$ and $f$ satisfying
\begin{align}
    \label{eq.lemma.learning.0}
    \mathbb E_{\vec x \sim \mathcal D}\left[\epsilon_{f,\vec x}(S)\right] \leq  L_f(r)
\end{align}
\end{lemma}

Before we prove Lemma~\ref{lemma.learning}, let us discuss a little what it says. 
We would like to learn explanations for inputs $\vec x$ with little precision error $\epsilon_{f,\vec x}(S)$. 
Since we have made no assumption whatsoever on $f$ which we treat as a black box, it is hard to see how to come up with a rule that does so. 
An easier task might be learning a representation of $f$ as a rule in $\mathcal R_{f,\vec x,k}$ in the hope that this rule could also serve as an explanation. 
What Lemma~\ref{lemma.learning} says is that this hope is well founded: in expectation, the precision error of the learned rule is not worse than its loss. 
Applying a standard Markov inequality argument, this can be lifted to show that for a good fraction of the inputs $\vec x$, 
the learned rule is a good explanation if the loss is small. So in the rest of this section, 
we will then show how to learn rules that approximate the function $f$ well and use these rules as explanations.

\begin{proof} 
    As $S$ is defined by the set of indices $j$ such that $(X_j = x_j)$ is a literal of $\Body{R}$, 
    it is a $k$-explanation for $f$ and $\vec x$ because $R$ is a $k$-rule. 
    Furthermore, we know by union bound that 
    \begin{align}  
        \mathbb P_{\vec z \sim \mathcal D} [f(\vec z) \neq f(\vec x) \mid \vec z_S = \vec x_S] 
        & \leq \mathbb P_{\vec z \sim \mathcal D} [f(\vec z) \neq r(\vec z) \mid \vec z_S = \vec x_S] \label{eq.lemma.learning.1} \\
        & + \mathbb P_{\vec z \sim \mathcal D} [r(\vec z) \neq r(\vec x) \mid \vec z_S = \vec x_S] \label{eq.lemma.learning.2} \\
        & + \mathbb P_{\vec z \sim \mathcal D} [r(\vec x) \neq f(\vec x) \mid \vec z_S = \vec x_S] \label{eq.lemma.learning.3} 
    \end{align}
    Consider the probability (\ref{eq.lemma.learning.2}). 
    By construction, the prediction made by $R$ depends only on the indices in $S$. 
    So $r(\vec z) = r(\vec x)$ whenever $\vec z_S = \vec x_S$, and hence 
    \begin{align*}
        P_{\vec z \sim \mathcal D} [r(\vec z) \neq r(\vec x) \mid \vec z_S = \vec x_S] &= 0
    \end{align*}
    Now, observe that the event in (\ref{eq.lemma.learning.3}) is independent of the choice of $\vec z$. 
    Since by definition all hypotheses in $\mathcal R_{f,\vec x,k}$ agree with $f$ on $\vec x$, we have $r(\vec x) = f(\vec x)$, so 
    \begin{align*}
        \mathbb P_{\vec z \sim \mathcal D} [r(\vec x) \neq f(\vec x) \mid \vec z_S = \vec x_S] &=  0
    \end{align*}
    Consider finally the right-hand side of (\ref{eq.lemma.learning.1}). Using the fact that $\vec x$ is chosen randomly according to $\mathcal D$, we have 
    \begin{align*}
        \mathbb E_{\vec x \sim \mathcal D} \left[\mathbb P_{\vec z \sim \mathcal D} [f(\vec z) \neq r(\vec z) \mid \vec z_S = \vec x_S]\right] 
        &= \mathbb P_{\vec z \sim \mathcal D} [f(\vec z) \neq r(\vec z)]     
    \end{align*}
    Thus, taking expectations on both sides of (\ref{eq.lemma.learning.1}) and plugging in (\ref{def.precision}) and (\ref{eq.true.loss}), 
    it follows that
    \begin{align*}
        \mathbb E_{\vec x \sim \mathcal D}\left[\epsilon_{f,\vec x}(S)\right] \leq  L_f(r)
    \end{align*}
    which concludes the proof. 
\end{proof}

\subsection{Learning Rules via Constraint Optimization}
\label{sec.cop}

In light of Lemma~\ref{lemma.learning}, our framework's objective is to learn a hypothesis in $\mathcal R_{f,\vec x,k}$.
Using the fact that each hypothesis in this class is representable by a rule whose body is included in the set of literals $\{X_j = x_j\}_{j = 1}^d$,
we will show that this objective is equivalent to the problem of \emph{(proper) learning monotone $k$-monomials}. 
Monomials are known to be efficiently learnable in the \emph{realizable} case where the target function $f$ can be represented by a monomial \cite{Valiant.ACM.1984}. 
However, in the \emph{agnostic} case, where $f$ is an arbitrary classifier, learning monotone monomials is \NP-hard \cite{Kearns.ML.1994} 
and hard to approximate \cite{Feldman.SIAM.2009}. Nonetheless, monotone monomials are one of the 
simplest concept classes in machine learning, endowed with a low sample complexity. 
Hence, by framing the learning task as a constraint optimization problem, a constraint solver should be able to efficiently 
learn short rules in $\mathcal R_{f,\vec x,k}$ within a reasonable amount of time.

In formal terms, the learning problem we wish to solve via constraint optimization is the following: 
given an instance $\vec x \in \{0,1\}^d$ to explain, a size limit $k$, and a set of samples $\{(\vec z_i, f(\vec z_i))\}_{i=1}^m$ 
generated from the example oracle $\EX(f,\mathcal D)$, find a hypothesis $r \in \mathcal R_{f,\vec x,k}$
that minimizes the empirical loss
\begin{align}
\label{eq.emp.loss}
    \hat{L}_f(r) \defeq \frac{1}{m} \sum_{i=1}^m \mathbbm{1}[r(\vec z_i) \neq f(\vec z_i)].
\end{align}

We first demonstrate that this task is equivalent to learning a monotone $k$-monomial.  
To this end, recall that a monotone monomial is a conjunction $C$ of variables over $\{X_j\}_{j=1}^d$.
A hypothesis $c: \{0,1\}^d \rightarrow \{0,1\}$ is representable by $C$ 
if for any vector $\vec u \in \{0,1\}^d$, we have $c(\vec u) = 1$ if and only if $u_j = 1$ 
for each $j \in [d]$. Let $\mathcal C$ denote the class of all hypotheses representable 
by monotone monomials, and let $\mathcal C_k$ denote the restriction of $\mathcal C$ 
to the hypotheses representable by monotone $k$-monomials.

Now, for two Boolean vectors $\vec a,\vec b \in \{0,1\}^d$, let $\vec a \otimes \vec b$ 
denote the bitwise equivalence between $\vec a$ and $\vec b$, 
given by $(a_i \otimes b_i) = \mathbbm 1[a_i = b_i]$ for each $i \in [d]$. 
Based on this operation, we map each sample $(\vec z_i, f(\vec z_i))$ into an example  
$(\vec u_i, v_i)$ where $\vec u_i = \vec z_i \otimes \vec x$ 
and $v_i = f(\vec z_i) \otimes f(\vec x)$. With this transformation in hand, consider 
the task of finding a function $c \in \mathcal C_k$ described by a $k$-monotone monomial that minimizes
\begin{align}
\label{eq.mon.loss}    
    \hat{L}(c) \defeq \frac{1}{m} \sum_{i=1}^m \mathbbm{1}[c(\vec u_i) \neq v_i]
\end{align}

\begin{lemma}
    \label{lemma.equivalence} 
    For a black-box model $f$ and an instance $\vec x$, let $\{(\vec z_i, f(\vec z_i))\}_{i=1}^m$ be a set of samples 
    generated by $\EX(f,\mathcal D)$, and let $\{(\vec u_i, v_i)\}_{i=1}^m$ be the corresponding training set defined as above.
    Then, for any $k \in \mathbb N$
    \begin{align*}
        \min_{c \in \mathcal C_k} \hat{L}(c) = \min_{r \in \mathcal R_{f,\vec x,k}} \hat{L}_f(r)
    \end{align*}
    Furthermore, given a monomial representation of a minimizer of $\hat{L}(\cdot)$ in $\mathcal C_k$,
    a rule representation of a minimizer of $\hat{L}_f(\cdot)$ in $\mathcal R_{f,\vec x,k}$ can be found in $\mathcal O(k)$ time.
\end{lemma} 

\begin{proof}
    Consider the mapping $C \mapsto R$, where $C$ is a monotone $k$-monomial, 
    and $R$ is a rule whose body is the set of literals $X_j = x_j$ such that $X_j$ is a variable in $C$,
    and whose head is the label $f(\vec x)$.
    Clearly, $R$ is a $k$-rule that can be constructed in $\mathcal O(k)$ time.
    Consider also the mapping $C \mapsto c$ which associates a hypothesis $c \in \mathcal C_k$ to a monotone $k$-monomial.
    Observe that $c$ is uniquely determined from $C$.
    The same holds for the mapping $R \mapsto r$ that associates a hypothesis $r \in \mathcal R_{f,\vec x,k}$ to a $k$-rule $R$.
    Therefore, both $C \mapsto c$ and $R \mapsto r$ are bijective. 
    This, together with the fact that $C \mapsto R$ is also bijective, implies a one-to-one correspondence between $\mathcal C_k$ and $\mathcal R_{f,\vec x,k}$.
    
    Now, let $C$ be a monotone $k$-monomial, and let $c \in \mathcal C_k$ and $r \in R_{f,\vec x,k}$ be the hypotheses obtained from the mappings 
    $C \mapsto c$, $C \mapsto R$, and $R \mapsto r$. Any sample $(\vec u_i, v_i)$ can be mapped to a sample $(\vec z_i, f(\vec z_i))$ 
    using $\vec z_i = \vec u_i \otimes \vec x$ and $f(\vec z_i) = v_i \otimes f(\vec x)$.
    Based on this bijective transformation, it follows that $c(\vec u_i) = v_i$ if and only if $r(\vec z_i) = f(\vec z_i)$.      
    Therefore, $\hat{L}(c) = \hat{L}_f(r)$, and since the feasible sets $\mathcal C_k$ and $\mathcal R_{f,\vec x,k}$ are in one-to-one correspondence,
    the optimal values for $\hat{L}(\cdot)$ and $\hat{L}_f(\cdot)$ are equal.  
\end{proof}

Based on these considerations, we are now in a position to provide a constraint encoding of our learning problem.
Recall that a \emph{Constraint Optimization Problem} (COP) consists of a set of variables, each associated with an integer domain, 
a set of constraints, and an objective function to minimize. As stated by Lemma~\ref{lemma.equivalence}, we can focus on the task 
of minimizing (\ref{eq.mon.loss}), since we can derive in linear time a representation of a minimizer of (\ref{eq.emp.loss}).

\begin{table}[t!]
    \begin{center}    
        \begin{tabularx}{\dimexpr\textwidth}{p{0.6cm} p{3.2cm}  >{\raggedright\arraybackslash}X}
            \toprule
            (i)   & Feature variables:      & $\{X_j\}_{j=1}^d$ with domain $\{0,1\}$\\ 
            (ii)  & Channeling variables:   & $\{U_i\}_{i=1}^m$ with domain $\{0,1\}$\\ 
            (iii) & Cardinality variable:   & $K$ with domain $\{0,\cdots,k\}$\\
            \midrule 
            (iv)  & Cardinality constraint: & $\left(\sum_{j=1}^d X_j \leq K\right)$\\
            (v)   & Positive constraints:   & $U_i \leftrightarrow \mathbbm 1\left[\sum_{j=1}^d X_{j} z_{ij} < K\right]$, $\forall i \in [m]$ s.t. $v_i = 1$  \\
            (vi)  & Negative constraints:   & $U_i \leftrightarrow \mathbbm 1\left[\sum_{j=1}^d X_{j} z_{ij} \geq K\right]$, $\forall i \in [m]$ s.t. $v_i = 0$  \\
            \midrule 
            (vii) & Objective function:     & $\min \left(\sum_{i=1}^m U_i\right)$\\
            \bottomrule
        \end{tabularx}   
    \end{center}
    \caption{COP for learning monotone $k$-monomials.}
    \label{table.cop}
\end{table}

The COP associated with this optimization task is described in Table~\ref{table.cop}. 
For each feature $j \in [d]$, we use one Boolean variable $X_j$, and for each sample index $i \in [m]$, 
we use one Boolean variable $U_i$.
Additionally, we use an integer variable $K$ with domain $\{0,\cdots,k\}$. 
To restrict the size of the monotone rule, we use a cardinality constraint (iv). 
We also use two series of channeling constraints (v - vi) to encode the loss of the rule on each example.
Based on these channeling constraints, the empirical loss of the rule is encoded in the objective function (vii).
The set of variables $X_j$ such that $X_j = 1$ in the minimizer of (vii) returned by the constraint solver 
defines the monotone $k$-monomial $C$, and the set of indices of these variables yields the explanation $S$.

\subsection{From Learning to Explanations}

We can now combine the results obtained so far. 
Informally, Lemma~\ref{lemma.learning} tells us that any classifier $r \in \mathcal R_{f,\vec x,k}$ 
achieving a small loss $L_f(r)$ can provide a $k$-explanation $S$ for $\vec x$ and $f$ with a small expected precision error.
In addition, Lemma~\ref{lemma.equivalence} shows us that the task of learning $r$ is equivalent to learning a monotone $k$-monomial, 
which can be achieved through constraint optimization. 
The next theorem establishes a connection between these results using standard tools in learning theory.

\begin{theorem}
    \label{theorem.main}
    Let $\mathcal D$ be a probability distribution over $\{0,1\}^d$ and $f: \{0,1\}^d \rightarrow \{0,1\}$ be a black-box model accessed through an example oracle $\EX(f,\mathcal D)$.
    Then, for any data instance $\vec x$ drawn at random from $\mathcal D$,
    any $k \in \mathbb N$, and any $\varepsilon,\delta \in (0,1)$, a $k$-explanation $S$ for $\vec x$ and $f$ 
    satisfying, with probability at least $1 - \delta$,  
    \begin{align}
        \label{eq.theorem}
        \mathbb E_{\vec x \sim \mathcal D}\left[\epsilon_{f,\vec x}(S)\right] \leq \min_{r \in \mathcal R_{f,\vec x,k}} L_f(r)  + \varepsilon   
    \end{align}
    can be found using $\mathcal O\left(\displaystyle\frac{k \ln(d) + \ln(1/\delta)}{\varepsilon^2}\right)$ calls to $\EX(f,\mathcal D)$.
\end{theorem}
\begin{proof}
    Let $S$ be a $k$-explanation returned by the constraint solver for the COP described in Table~\ref{table.cop}.
    Let $c$ be the hypothesis in $\mathcal C_k$ represented by the monotone $k$-monomial $C = \{X_j: j \in S\}$.
    Since the solver returns an optimal solution for the COP, we know that $c$ is a minimizer of the empirical loss $\hat L(\cdot)$ over $\mathcal C_k$.
    Furthermore, by Lemma~\ref{lemma.equivalence}, $C$ can be converted in linear time into a $k$-rule representation of a minimiser $r$ of the empirical loss 
    $\hat L_f(\cdot)$. Now, using the fact that $\size{\mathcal R_{f,\vec x,k}} \leq d^k$, if the number $m$ of calls to $\EX(f,\mathcal D)$ satisfies 
    \begin{align*}
        m \geq \frac{2}{\varepsilon^2} \left[ k \ln(d) + \ln \Bigl(\frac{2}{\delta}\Bigr) \right]
    \end{align*}
    then, by application of Corollary 4.6 in \cite{ShalevShwartz.Book.2014}, we have
    \begin{align*}
    L_f(r) \leq  \min_{r \in \mathcal R_{f,\vec x,k}} L_f(r)  + \varepsilon
    \end{align*}
    with probability $1 - \delta$. Plugging this bound into (\ref{eq.lemma.learning.0}) yields the result. 
\end{proof}

\paragraph{Note.} It is important to keep in mind that the bound (\ref{eq.theorem}) is only true in expectation. 
It would be comforting to know that the same bound holds for \emph{every} possible input $\vec x$. 
However, this is not the case. The issue is that the error made by a rule in approximating 
a function $f$ may be amplified when considering the precision error $\epsilon_{f,\vec x}(S)$,
which is a probability that is conditioned on a restricted subspace of instances.
In a nutshell, Theorem~\ref{theorem.main} circumvents the hardness results in \cite{Waeldchen.JAIR.2021},
because our constraint-based approach is expected to provide good approximations 
only for a fraction $(1 - \delta)$ of random instances $\vec x$ to explain.

\section{Experiments}

In order to validate the effectiveness of our approach, we have considered various model-agnostic explanation tasks.
The code was written using the \texttt{Python} language. 
Our experiments have been conducted on a Quad-core Intel XEON X5550 with 32GB of memory. 

\subsection{Experimental Setup}

We conducted experiments on $B = 25$ tabular datasets (or benchmarks) from the \texttt{OpenML} repository.
All datasets are classification tasks including both numerical and categorical attributes. 
To convert these raw attributes into interpretable binary features, 
we used a \texttt{Scikit-Learn} implementation of the $K$-bins discretization method with $3$ bins per attribute. 
The resulting interpretable instances have an average dimension $d$ of $60$, ranging from $12$ to $352$.  

For each benchmark $b \in [B]$, an explanation task is defined by a tuple $(f,\vec x,k)$, 
where $f$ represents a black-box model implemented by a neural network. 
For our experiments, we utilized a \texttt{Scikit-Learn} implementation of the 
multi-layer perceptron classifier, using default parameters. 
As usual, we trained $f$ on the training set of the benchmark $b$, 
and evaluated its accuracy on the test set of $b$. 
Each instance $\vec x$ to explain was randomly selected 
from the test set of $b$. 
For multi-label classification tasks, the label $f(\vec x)$ of $\vec x$ was set to $1$, 
and for any instance $\vec z$ such that $f(\vec z) \neq f(\vec x)$, the label $f(\vec z)$ of $\vec z$ 
was set to $0$. Finally, we used $k \in \{1,\cdots,7\}$ for the size limits.

The empirical precision error $\hat{\epsilon}_{f,\vec x}(S)$
of an explanation $S$ for $\vec x$ and $f$ was measured using 
$m = 1000$ labeled samples $(\vec z,f(\vec z))$, where 
$\vec z$ was generated uniformly at random over $\{0,1\}^d$.
The performance of explanation algorithms on a benchmark $b$ 
was measured by selecting $100$ random instances $\vec x$, 
and averaging the empirical precision error $\hat{\epsilon}_{f,\vec x}(S)$.

In our experiments, we explored two different encodings of the constraint optimization problem, described in Section~\ref{sec.cop}. 
The first encoding, called \textsc{Cop}, includes all constraints listed in Table~\ref{table.cop}. 
Thus, \textsc{Cop} uses the channeling constraints (v - vi) for both positive and negative examples in the sample set. 
The second encoding, named \textsc{Sat}, excludes the positive channeling constraints (v). 
Although \textsc{Sat} is no longer guaranteed to return an empirical minimizer of (\ref{eq.true.loss}), 
it is smaller than \textsc{Cop} because it only considers negative examples. 
Furthermore, the constraints (vi) can be broken down into propositional clauses, 
which speeds up the search for optimal solutions. 
We used the \texttt{CP-SAT} solver from \texttt{OR-Tools} \cite{Perron.2024}, 
with a 60-second timeout per explanation task. 
When the timeout is reached, the solver returns the current best (sub-optimal) solution.

Both constraint encodings were compared with the state-of-the-art \textsc{Anchors} approach.
As mentioned above, \textsc{Anchors} is a beam search method that starts from an empty rule and then adds literals iteratively, estimating their quality using a stochastic bandit algorithm. 
For our experiments, we followed the default parameters recommended in \cite{Ribeiro.AAAI.2018}; we set the number of beams to $10$, and the tolerance and precision of the bandit algorithm to $99\%$ and $99.5\%$, respectively.

\subsection{Experimental Results}

\setlength{\tabcolsep}{6pt}
\begin{table*}[t]
    \begin{center}
    \begin{adjustbox}{width=\linewidth,center}  
    \begin{tabular}{lr r rrr r rr}
        \toprule
        \multicolumn{2}{c}{Benchmark} & & \multicolumn{3}{c}{$\hat{\epsilon}_{f,\vec x}(S)$} & & \multicolumn{2}{c}{Time (s)} \\
        \cmidrule{1-2} \cmidrule{4-6} \cmidrule{8-9}
        Name & ID &&                          \textsc{Cop}        & \textsc{Sat}      & \textsc{Anchors}  && \textsc{Cop} & \textsc{Sat} \\
        \midrule
        \emph{ada prior}	   & $1037$  &&	$\mathbf{0.14 (\pm 0.20)}$ & $0.19 (\pm 0.24)$ & $0.43 (\pm 0.11)$ &&  $60.06$ & $12.02$ \\			
        \emph{adult census}	   & $1119$  &&	$\mathbf{0.19 (\pm 0.21)}$ & $0.24 (\pm 0.26)$ & $0.44 (\pm 0.12)$ &&  $60.05$ & $18.02$ \\
        \emph{aids}	           & $346$	 && $0.05 (\pm 0.05)$ & $0.04 (\pm 0.04)$ & $0.59 (\pm 0.41)$          &&  $01.73$ & $00.06$ \\	
        \emph{australian}      & $40981$ && $\mathbf{0.14 (\pm 0.12)}$ & $0.20 (\pm 0.13)$ & $0.51 (\pm 0.42)$ &&  $60.05$ & $07.13$ \\  
        \emph{bank marketing}  & $1461$	 && $\mathbf{0.05 (\pm 0.07)}$ & $0.13 (\pm 0.08)$ & $0.49 (\pm 0.11)$ &&  $60.04$ & $01.84$ \\
\emph{baseball}        & $185$   && $\mathbf{0.19 (\pm 0.15)}$ & $0.23 (\pm 0.15)$ & $0.77 (\pm 0.16)$ &&  $60.01$ & $01.96$ \\		
        \emph{calif. housing}  & $45578$ && $0.09 (\pm 0.08)	     $ & $0.10 (\pm 0.10)$ & $0.52 (\pm 0.45)$ &&  $28.97$ & $00.42$ \\		
        \emph{cardiotocography}& $1466$  && $\mathbf{0.35 (\pm 0.23)}         $ & $0.47 (\pm 0.29)$ & $0.74 (\pm 0.13)$ &&  $60.00$ & $36.33$ \\
        \emph{compas}          & $42192$ && $\mathbf{0.09 (\pm 0.07)}	     $ & $0.14 (\pm 0.09)$ & $0.59 (\pm 0.43)$ &&  $60.01$ & $00.16$ \\	
        \emph{credit approval} & $29$    &&	$\mathbf{0.12 (\pm 0.10)}$ & $0.18 (\pm 0.13)$ & $0.97 (\pm 0.30)$ &&  $60.02$ & $02.99$ \\		
        \emph{dermatology} 	   & $35$	 && $\mathbf{0.58 (\pm 0.38)}$          & $0.64 (\pm 0.32)$ & $0.75 (\pm 0.27)$ &&  $60.02$ & $31.30$ \\		
        \emph{diabetes}        & $37$    && $0.08 (\pm 0.06)$          & $0.10 (\pm 0.07)$ & $0.62 (\pm 0.37)$ &&  $31.91$ & $00.03$ \\
\emph{hepatitis}	   & $55$    && $0.14 (\pm 0.15)$          & $0.16 (\pm 0.17)$ & $0.28 (\pm 0.38)$ &&  $59.97$ & $06.75$ \\ 	
        \emph{heart disease}   & $43944$ && $0.17 (\pm 0.12)$          & $0.19 (\pm 0.10)$ & $0.32 (\pm 0.41)$ &&  $60.01$ & $01.79$ \\		
        \emph{iris}            & $61$	 && $0.02 (\pm 0.04)$          & $0.03 (\pm 0.06)$ & $0.41 (\pm 0.27)$ &&  $00.33$ & $00.06$ \\ 	
        \emph{lymph}           & $10$ 	 && $\mathbf{0.19 (\pm 0.13)}$          & $0.24 (\pm 0.14)$ & $0.32 (\pm 0.26)$ &&  $60.01$ & $05.26$ \\		
\emph{mushroom}        & $43922$ && $\mathbf{0.18 (\pm 0.09)}$          & $0.28 (\pm 0.12)$ & $0.53 (\pm 0.36)$ &&  $60.04$ & $00.85$ \\	
		\emph{nomao}           & $1486$  && $\mathbf{0.22 (\pm 0.14)}$ & $0.35 (\pm0.16)$ & $0.30 (\pm 0.31)$ && $60.06$ & $15.53$\\
        \emph{postoperative}   & $34$    && $\mathbf{0.11 (\pm 0.17)}$          & $0.17 (\pm 0.17)$ & $0.46 (\pm 0.29)$ &&  $26.92$ & $00.10$ \\
        \emph{primary tumor}   & $171$   && $0.66 (\pm 0.39)$          & $0.69 (\pm 0.35)$ & $0.96 (\pm 0.11$) && $60.01$ & $32.95$ \\ 
        \emph{shuttle}         & $172$	 && $0.09 (\pm 0.08)$          & $0.10 (\pm 0.08)$ & $0.30 (\pm 0.35)$ &&  $01.88$ & $00.04$ \\	
        \emph{stock}           & $841$   && $0.08 (\pm 0.07)$          & $0.11 (\pm 0.08)$ & $0.45 (\pm 0.40)$ &&  $54.91$ & $00.45$ \\    
        \emph{tae}  	       & $48$	 && $\mathbf{0.23 (\pm 0.21)}$ & $0.30 (\pm 0.28)$ & $0.38 (\pm 0.36)$ &&  $00.46$ & $00.07$ \\	
        \emph{veteran}         & $719$   && $0.17 (\pm 0.25)$          & $0.19 (\pm 0.29)$ & $0.30 (\pm 0.44)$ &&  $08.83$ & $00.11$ \\
        \emph{vote}            & $56$    && $0.07 (\pm 0.08)$          & $0.06 (\pm 0.06)$ & $0.30 (\pm 0.19)$ &&  $60.01$ & $00.16$ \\
\bottomrule
    \end{tabular}
    \end{adjustbox}
    \end{center}
  \caption{Experimental results on 25 benchmarks, using $m = 1000$ and $k = 5$.}
  \label{tab.results}
\end{table*}

An overview of our results on the 25 benchmarks, for $k=5$, is presented in Table~\ref{tab.results}. 
The first two columns of the table provide the name of the dataset $b$ and its corresponding \texttt{OpenML} identifier. 
The next three columns present the empirical precision errors of the explanations obtained through the encodings \textsc{Cop}, \textsc{Sat}, and the \textsc{Anchors} method. 
The last three columns show the average runtimes (in seconds) of the constraint solver for the encodings \textsc{Cop} and \textsc{Sat}. 
We did not report the runtimes of \textsc{Anchors} since they were below one second.

In light of those results, we can confidently conclude that our constraint-based approach outperforms the \textsc{Anchors} method in terms of empirical precision error across all benchmarks. 
Similar to \textsc{Cop} and \textsc{Sat}, \textsc{Anchors} generates an \emph{if-then} rule to explain the prediction $f(\vec x)$ of an input instance $\vec x$. 
However, the head of the rule produced by \textsc{Anchors} can sometimes differ from $f(\vec x)$, which explains its poor performance on datasets such as, for example, 
\emph{aids}, \emph{baseball}, \emph{credit approval} and \emph{diabetes}. 
On the other hand, the rules generated by \textsc{Cop} and \textsc{Sat} always agree with $f$ on $\vec x$ since they represent hypotheses in $\mathcal R_{f,\vec x,k}$.

Out of the 25 benchmarks, \textsc{Cop} shows a precision error of less than $20\%$ for 20 of them. 
For 9 of these benchmarks, the precision error is even less than $10\%$. 
This suggests that using a constraint-based approach to learn an optimal rule in $\mathcal R_{f,\vec x,k}$ 
can provide highly accurate explanations. 
However, for certain datasets like \emph{cardiotocography}, \emph{dermatology}, and \emph{primary tumor}, 
the performance of \textsc{Cop} is not as good, which indicates that the black-box model $f$ is sometimes too 
complex to be approximated by a rule in $\mathcal R_{f,\vec x,k}$.

\begin{figure}[t!]
    \centering
    \begin{subfigure}{0.48\textwidth}
        \includegraphics[width=\textwidth]{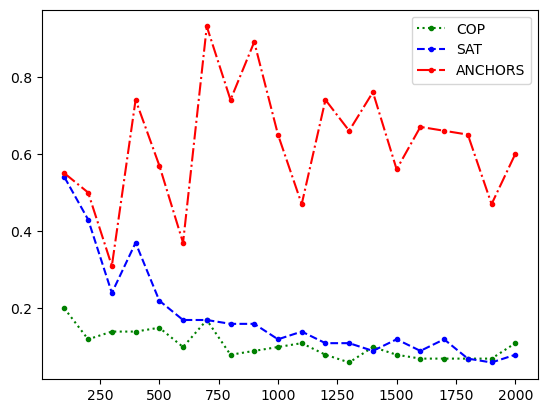}
        \caption{\emph{compas}}
        \label{fig.training.1}
    \end{subfigure}
    \hfill
    \begin{subfigure}{0.48\textwidth}
        \includegraphics[width=\textwidth]{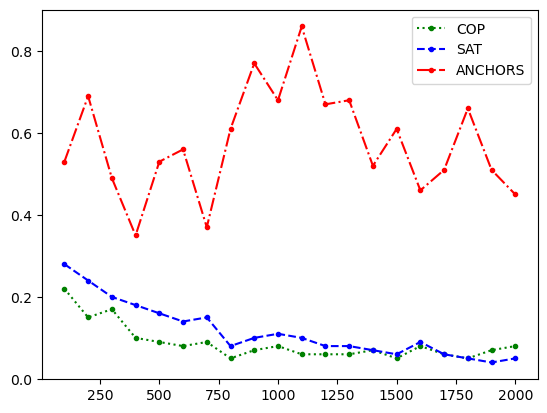}
        \caption{\emph{diabetes}}
        \label{fig.training.2}
    \end{subfigure}
    \hfill
    \begin{subfigure}{0.48\textwidth}
        \includegraphics[width=\textwidth]{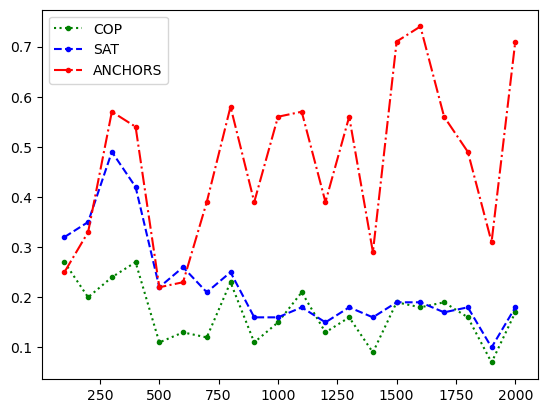}
        \caption{\emph{heart disease}}
        \label{fig.training.3}
    \end{subfigure}
    \hfill
    \begin{subfigure}{0.48\textwidth}
        \includegraphics[width=\textwidth]{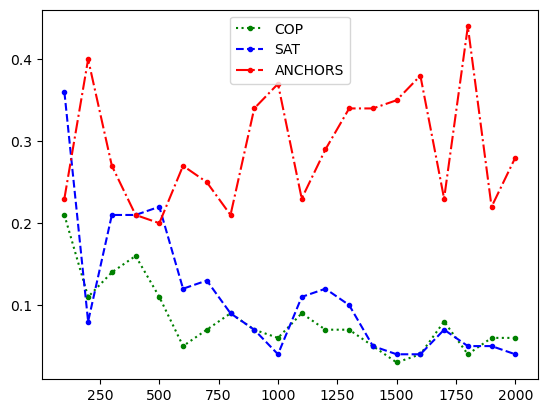}
        \caption{\emph{vote}}
        \label{fig.training.4}
    \end{subfigure}
                   
    \caption{Comparison of precision errors with increasing numbers of samples.}
    \label{fig.training}
\end{figure}

When comparing the encodings \textsc{Cop} and \textsc{Sat}, we can observe that \textsc{Cop} generally produces more precise explanations than \textsc{Sat}. 
The reason behind this is that \textsc{Cop} aims to minimize the loss $L_f(\cdot)$ over $\mathcal R_{f,\vec x,k}$, which results in more accurate explanations. 
However, \textsc{Sat} is much faster than \textsc{Cop}. 
Interestingly, \textsc{Sat} can achieve performances comparable to \textsc{Cop} in just a few seconds for 11 of the 25 benchmarks, 
and sometimes even in less than one second.

Figure~\ref{fig.training} displays the precision errors of explanations 
produced by \textsc{Cop}, \textsc{Sat}, and \textsc{Anchors} on four benchmarks. 
These precision errors are presented as a function of the number $m$ of samples used to learn the explanations. 
In particular, the plots illustrate how the performance of \textsc{Cop} and \textsc{Sat} improves as the number of samples increases, 
while \textsc{Anchors} exhibits much less stability.

Orthogonally, Figure~\ref{fig.sizes} shows the precision errors as a function of the size $k$ of explanations 
for the explainers \textsc{Cop}, \textsc{Sat}, and \textsc{Anchors}. 
The histograms indicate that both constraint encodings perform better than \textsc{Anchors} for all sizes ranging from $1$ to $7$. 
The performances of \textsc{Cop} and \textsc{Sat} are similar, except for the benchmarks \emph{compas} and \emph{heart disease}. 
In these cases, the precision of the suboptimal rules returned by \textsc{Sat} may degrade when increasing $k$, 
while the precision of the optimal rules found by \textsc{Cop} consistently improves with $k$.

\begin{figure}[t!]
    \centering
    \begin{subfigure}{0.46\textwidth}
        \includegraphics[width=\textwidth]{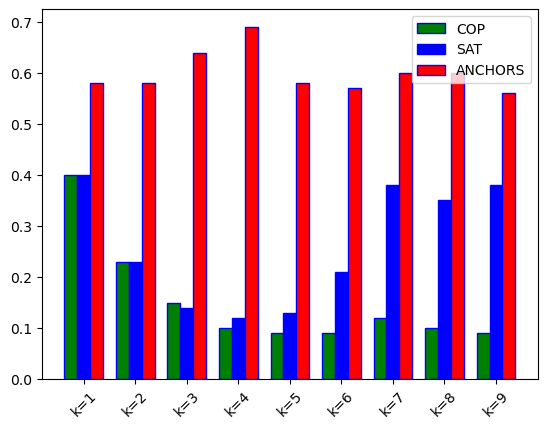}
        \caption{\emph{compas}}
        \label{fig.sizes.1}
    \end{subfigure}
    \hfill
    \begin{subfigure}{0.46\textwidth}
        \includegraphics[width=\textwidth]{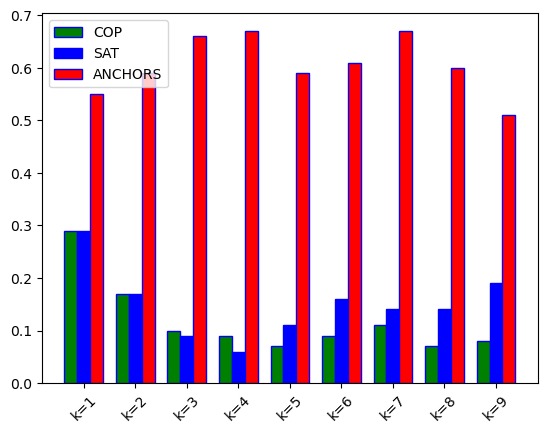}
        \caption{\emph{diabetes}}
        \label{fig.sizes.2}
    \end{subfigure}
    \hfill
    \begin{subfigure}{0.46\textwidth}
        \includegraphics[width=\textwidth]{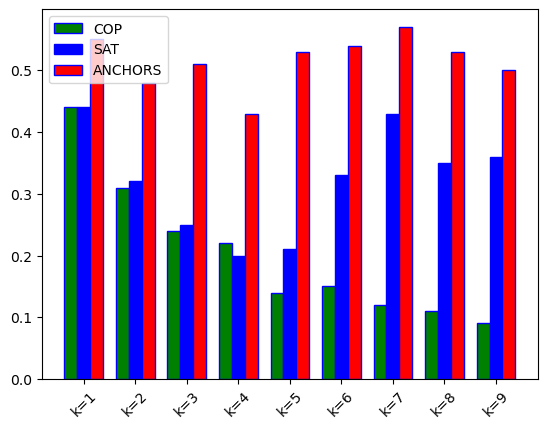}
        \caption{\emph{heart disease}}
        \label{fig.sizes.3}
    \end{subfigure}
    \hfill
    \begin{subfigure}{0.46\textwidth}
        \includegraphics[width=\textwidth]{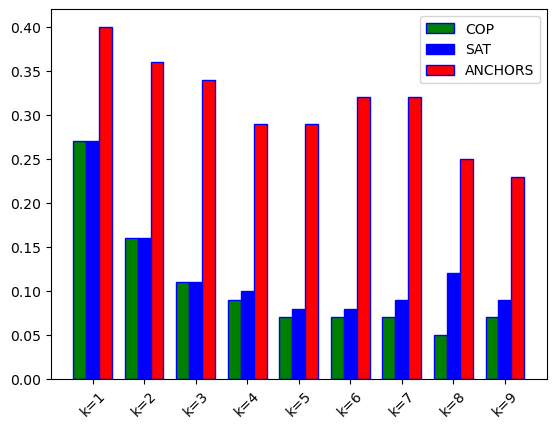}
        \caption{\emph{vote}}
        \label{fig.sizes.4}
    \end{subfigure}
                   
    \caption{Comparison of precision errors with increasing explanation sizes.}
    \label{fig.sizes}
\end{figure}

\section{Conclusion}

Interpretable Machine Learning faces the challenge of providing concise and precise explanations for black-box models. 
As most model-agnostic explainers are heuristic, we have proposed a learning framework that offers theoretical guarantees 
on the quality of explanations of a limited size. 
Our main result suggests that by learning an optimal monotone rule via constraint optimization, we can produce short explanations 
whose precision depends on the rule's ability to fit the black-box model. 
Evaluated on various benchmarks, our constraint encodings, \textsc{Cop} and \textsc{Sat}, outperform the state-of-the-art \textsc{Anchors} method.

\subsubsection*{Acknowledgements.}
Many thanks to the reviewers for their comments and suggestions. 
This work has benefited from the support of the AI Chair EXPEKCTATION (ANR-19- CHIA-0005-01) 
of the French National Research Agency. It was also partially supported by TAILOR, 
a project funded by EU Horizon 2020 research and innovation programme under GA No 952215.  

\bibliographystyle{splncs04}

\end{document}